\documentclass{ecai}
\usepackage{times}
\usepackage{latexsym}

\usepackage{amsmath, amssymb}
\usepackage{amsthm}
\usepackage{graphicx}

\theoremstyle{definition}
\newtheorem{dfn}{Definition}
\newtheorem{thm}{Theorem}

\newtheorem{ex}{Example}

\begin{document}

\title{Metric Learning for Ordered Labeled Trees with \textsl{\textbf{pq}}-grams}

\author{Hikaru Shindo\institute{Kyoto University, Japan, email: hikarushindo@iip.ist.i.kyoto-u.ac.jp}  
\and Masaaki Nishino\institute{NTT Communication Science Laboratories, Japan}
\and Yasuaki Kobayashi\institute{Kyoto University, Japan}
\and Akihiro Yamamoto\institute{Kyoto University, Japan}}


\maketitle
\bibliographystyle{ecai}

\begin{abstract}
Computing the similarity between two data points plays a vital role in many machine learning algorithms. 
Metric learning has the aim of learning a good metric automatically from data.
Most existing studies on metric learning for tree-structured data have adopted the approach of learning the tree edit distance.
However, the edit distance is not amenable for big data analysis because it incurs high computation cost. 
In this paper, we propose a new metric learning approach for tree-structured data with $pq$-grams.
The $pq$-gram distance is a distance for ordered labeled trees,  and has much lower computation cost than the tree edit distance.
In order to perform metric learning based on $pq$-grams, we propose a new differentiable parameterized distance, {\em weighted $pq$-gram distance}.
We also propose a way to learn the proposed distance based on Large Margin Nearest Neighbors (LMNN), which is a well-studied and practical metric learning scheme.
We formulate the metric learning problem as an optimization problem and use the gradient descent technique to perform metric learning.
We empirically show that the proposed approach not only achieves competitive results with the state-of-the-art edit distance-based methods in various classification problems, but also solves the classification problems much more rapidly than the edit distance-based methods.
\end{abstract}

\section{INTRODUCTION}

The performance of many machine learning algorithms depends on the way in which the distance or similarity between data points are measured~\cite{Kulis13}.
For instance, 
$k$-nearest neighbor classification~\cite{Cover67} decides the class label of a data point from those of its neighbors,
while Learning Vector Quantization (LVQ)~\cite{Kohonen1995} classifies each data point based on the closest prototype according to a given distance measure.
Clustering algorithms, such as K-Means \cite{Lloyd82}, rely on a given distance or similarity function for input data.
In order for these algorithms to perform accurately and efficiently, a metric that suits the given problem is necessary.
The metric should capture the characteristics of the datasets expected, i.e., a pair of data points from the same class is closer than a pair of points from different classes.
 
The objective of metric learning is to learn a good metric from data since handcrafting good metrics for specific problems is generally difficult~\cite{DBLP:journals/corr/BelletHS13}.
Metric learning has been an active research topic because of its applicability, i.e., any algorithm using a distance measure internally 
can benefit from its results~\cite{DBLP:journals/corr/BelletHS13, Kulis13}.
In metric learning, a training set consists of sets of pairs; {\em positive pairs} and {\em negative pairs}.
Metrics are learned by optimizing a loss function that makes positive pairs closer while separating negative pairs.
This enables us to improve the accuracy of various machine learning algorithms that depend on the metric.
Figure \ref{fig:abst} illustrates metric learning applied to tree-structured data, which is the focus of this paper.

Discrete structures, in particular tree structures, play a key role in several research domains, for instance, XML documents on the web, parse trees for computer programs and natural language, and glycan structures in bioinformatics.
Distance measures that exploit such structures have been extensively studied.
The tree edit distance~\cite{Tai:1979:TCP:322139.322143} is one of the common choices for processing tree-structured data.
Intuitively, the tree edit distance is measured by the number of operations needed to transform one input tree into another input tree.
The edit operations consist of deletion, insertion, and replacement of the nodes in the trees.
The tree edit distance is used in many research domains such as information extraction and bioinformatics~\cite{Jiang02, Reis04}. 
Computing the tree edit distance, however, is not scalable making it problematic for large-scale datasets.
The current best algorithm runs in $\mathcal{O}(n^3)$ time where $n$ is the number of nodes of the input trees~\cite{Demaine:2009:ODA:1644015.1644017}.
To overcome this issue, Augsten~\cite{Augsten:2008:PGD:1670243.1670247} proposed the $pq$-gram distance; it can be computed faster than the tree edit distance.
Computing the $pq$-gram distance of trees with $n$ nodes can be done in time $\mathcal{O}(n \log n)$ for fixed $p,q \in \mathbb{N}$.
Moreover, the $pq$-gram distance is known to approximate the fanout weighted tree edit distance, which is a weighted variant of the tree edit distance~\cite{Augsten:2008:PGD:1670243.1670247}.

Most existing studies of metric learning for tree-structured data use the tree edit distance in learning, i.e.,  
learned costs of the edit operations from examples~\cite{Bellet2012, Mokbel2015306, pmlr-v80-paassen18a}.
The edit distance is, however, expensive to compute and hence not suitable for large-scale datasets.

\begin{figure}[t]
    \centering
    \includegraphics[width=\linewidth]{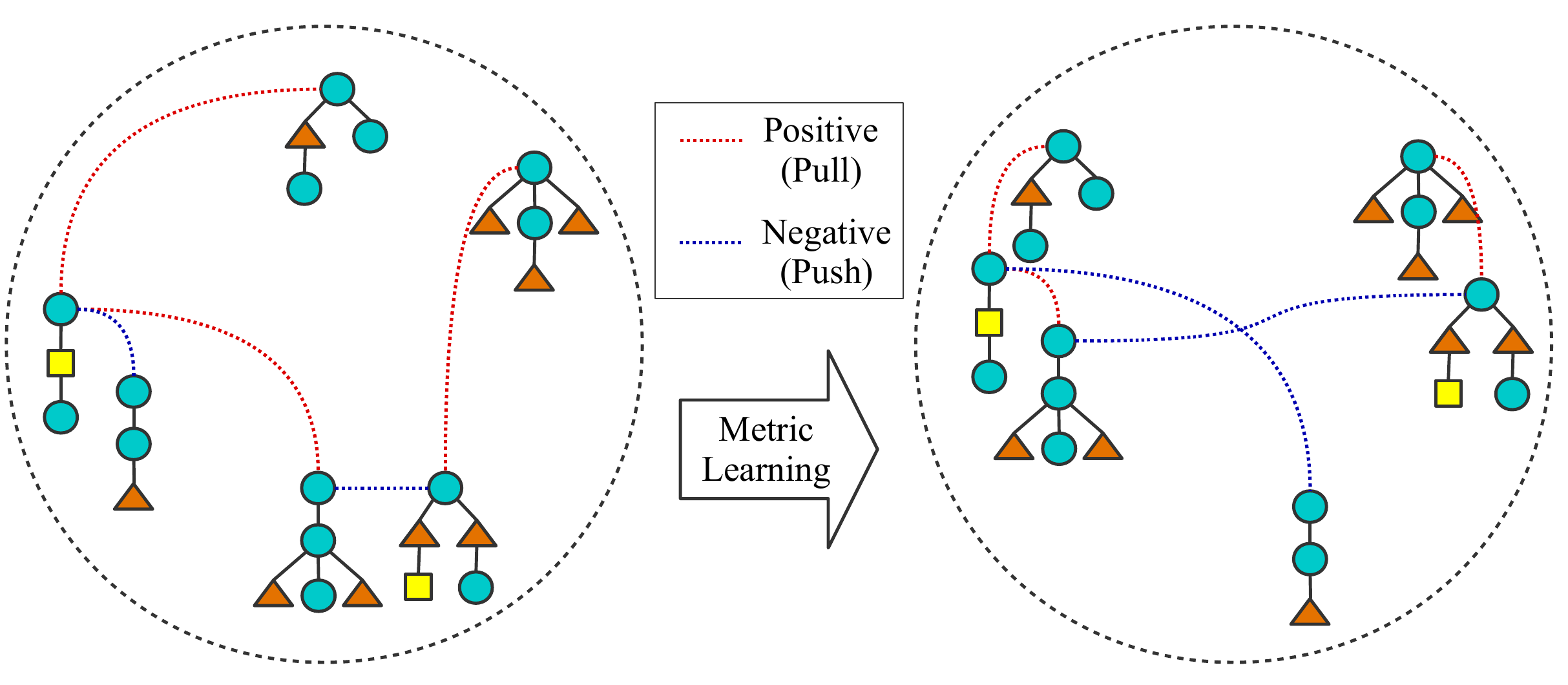}
    \caption{An illustration of metric learning applied to tree-structured data. The training examples are given as pairwise constraints. The red dotted lines represent positive pairs. The blue dotted lines represent negative pairs. Metrics are optimized to pull positive pairs closer and push negative pairs father apart.}
    \label{fig:abst}
    \vspace{-1.0em}
\end{figure}

In this paper, we propose a novel metric learning approach for tree-structured data that has the following features.
First, we propose a differentiable parameterized distance based on $pq$-grams, the {\em weighted $pq$-gram distance}, to achieve practical metric learning even for large-scale tree-structured data sets.
To make the distance function differentiable and always positive, we use the {\em softplus} function.
It enables us to learn the distance function by gradient descent techniques and retain the triangle inequality during the learning process.
Second, we also propose a way to learn the weighted $pq$-gram distance through Large Margin Nearest Neighbors (LMNN)~\cite{weinberger2009distance}, which is one of the most widely-used metric learning schemes.
Our proposed approach not only achieves results competitive with those of state-of-the-art edit distance-based methods~\cite{Bellet2012, pmlr-v80-paassen18a} in various classification problems, but also solves classification problems much faster than edit distance-based methods.
Third, our method is interpretable. 
Our approach shares some aspects with kernel methods~\cite{DBLP:Gartner03}, however,  different from kernel methods,
we do not implicitly cast data points into a high dimension space.
Moreover, our weight parameter indicates which tree substructures are important for classifying input trees.

The remainder of this paper is structured as follows.
We discuss related work on metric learning for structured data in Section 2.
In Section 3, we revisit the basic concepts of tree-structured data, the $pq$-gram distance, and the scheme of distance metric learning as background.
Section 4 describes our metric learning system in detail. 
Section 5 describes the experiments conducted on our methods, including accuracy and time comparisons.
We conclude in Section 6.

\section{RELATED WORK}
Our work is related to some machine learning research areas, especially metric learning and machine learning for structured data.

A pioneering study of metric learning learned Mahalanobis distance as an optimization problem~\cite{NIPS2002}.
Large Margin Nearest Neighbors (LMNN)~\cite{weinberger2009distance} was proposed in order to learn the Mahalanobis distance from nearest neighbors.
LMNN is often used because of its simplicity and efficiency ~\cite{kedem12, Kulis13, shibin10, mmlmnn}.
LMNN is also a well-studied metric-learning scheme. For instance, the relation between LMNN and Support Vector Machine has been pointed out in a unified view~\cite{huyen12}.
We apply the LMNN scheme to learn the distance between labeled ordered trees.

Almost all past studies on learning distances between trees employ edit distance, i.e., learning the costs of edit operations from examples.
Early work in learning edit distance is the stochastic approach~\cite{mcc05, sebban06, ristad98}.
Good Edit Similarity Learning (GESL)~\cite{Bellet2012} is a well-organized framework to learn the edit distance. GESL essentially optimizes $(\epsilon, \gamma, \tau)$-goodness~\cite{BalcanCOLT08, BalcanML08}, which guarantees its generalization performance.
Mokbel~\cite{Mokbel2015306} proposed a novel approach to learn simultaneously the edit costs for sequences and the Generalized Learning Vector Quantization (GLVQ)~\cite{Sato:1995:GLV:2998828.2998888} model.
Paa{\ss}en \cite{pmlr-v80-paassen18a} proposed to learn embeddings of the tree node labels while learning the GLVQ model. This approach is called Embedding Edit Distance Learning (BEDL) and succeeds in learning the edit distance flexibly from examples.
These works blazed a trail in the field of metric learning for structured data.
However, all of these methods incur high computation cost since they essentially compute the edit distance between trees.
Our method uses the $pq$-gram distance rather than the edit distance in learning the parameters.
This is a key difference between past studies and our approach.

In the field of machine learning for structured data, the kernel method is an active research topic~\cite{DBLP:Gartner03}.
For instance, Kuboyama~\cite{kuboyama2007} proposed a kernel function that is computed from $q$-grams of trees.
Tree kernels have been applied in many domains, such as Natural Language Processing and Bioinformatics~\cite{mos2006, yamanishi07}.
Kernel methods, however, lack interpretability since they implicitly cast data points into a high dimension space.
Moreover, the kernel matrix must be positive semidefinite, but this constraint does not suit some problems~\cite{Schleif2015}. 

\section{BACKGROUND}
\begin{figure*}[t]
    \centering
    \includegraphics[width=\linewidth]{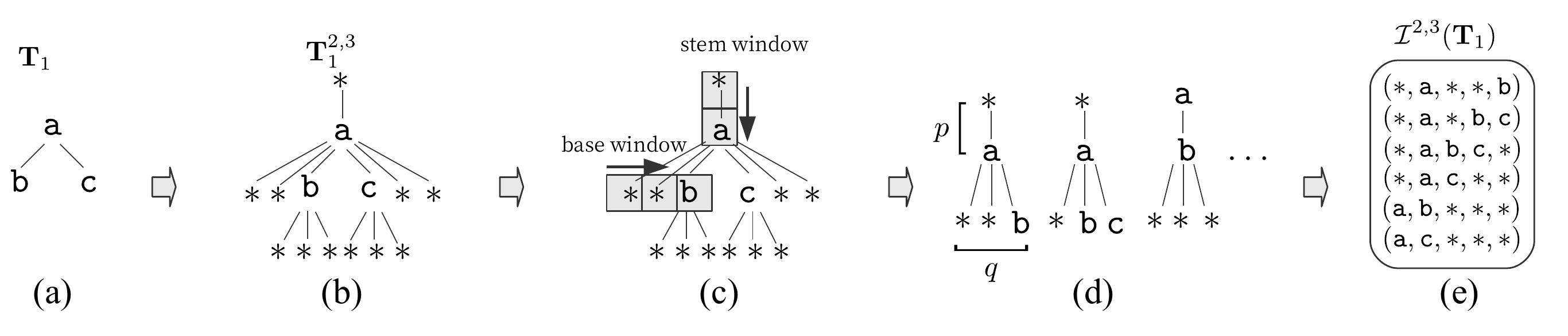}
    \vspace{-1.0em}
    \caption{Overview of extracting $pq$-grams with $p=2$ and $q=3$ from tree ${\bf T}_1$. (a) A rooted labeled ordered tree ${\bf T}_1$. ~(b) The $pq$-extended tree of ${\bf T}_1$, is obtained by inserting dummy nodes as in Def. \ref{def:extend}. The dummy nodes are represented by ${\tt *}$. ~(c) Sliding two windows $pq$-extended tree. The size $p$ window (stem window) moves vertically and the size $q$ window (base window) moves horizontally. ~(d) The multiset of extracted $pq$-grams of ${\bf T}_1$. ~(e) The $pq$-gram index of ${\bf T}_1$.
    Since the shape of pq-grams depends only on the values of $p$ and $q$, we can represent the multiset of pq-grams as a multiset of tuples of length $p + q$. }
    \label{fig:pq_flow}
\end{figure*}
In this section, we define the basic concepts of tree-structured data and the $pq$-gram distance following~\cite{Augsten:2008:PGD:1670243.1670247}.
We also review the general concept of metric learning following~\cite{DBLP:journals/corr/BelletHS13, Kulis13} and a metric learning algorithm following \cite{DBLP:journals/corr/BelletHS13, weinberger2009distance}.
\subsection{Preliminaries}
{\em Tree} ${\bf T}$ is a directed, acyclic, connected, non-empty graph with node set $N({\bf T})$ and edge set $E({\bf T})$. 
An {\em edge} is an ordered pair $(p, c)$, where $p, c \in N({\bf T})$ are nodes, and $p$ is the {\em parent} of $c$. 
A node can have at most one parent, and nodes with the same parent are {\em siblings}. 
Total order $<$ is defined on each group of sibling nodes. 
Two siblings $s_1, s_2$ are {\em contiguous} iff $s_1 < s_2$ and they have no sibling $x$ such that $s_1 < x < s_2$. 
Node $c$ is the $i$-th child of $p$ iff $i = |\{x \in N({\bf T})| (p, x) \in E({\bf T}), x \leq c\}|$. 
The node with no parent is the root node, denoted by $root({\bf T})$, and a node without children is a {\em leaf}. 
Each node $v$ has a {\em label}, $\lambda (v) \in \Sigma$, where $\Sigma$ is a finite alphabet. 
In what follows, such trees are called {\em ordered labeled} trees.
We write, in recursive style, tree ${\bf T}$ as $x(c_1, \ldots, c_n)$ where $x = \lambda(root({\bf T}))$ and $c_1, \ldots c_n$ is a list of subtrees whose root is a child  of the root node.

\subsection{$pq$-gram}
Intuitively, the $pq$-grams of a tree are all subtrees with a specific shape.
Parameters $p,q \in \mathbb{N}$ define $pq$-gram shape.
 To ensure that each node of the tree appears in at least one $pq$-gram, we extend the tree with dummy nodes.
\begin{dfn}[$pq$-Extended Tree]
Let ${\bf T}$ be a tree, and $p > 0$ and $q > 0$  be two integers. The {\em $pq$-extended tree}, ${\bf T}^{p,q}$, is constructed from ${\bf T}$ by 
(i) adding $p-1$ ancestors to the root node, 
(ii) inserting $q-1$ children before the first and after the last child of each non-leaf node, 
and (iii) adding $q$ children to each leaf of ${\bf T}$. All newly inserted nodes are dummy nodes that do not occur in ${\bf T}$ and have a special label ${\tt \ast} \notin \Sigma$.
\label{def:extend}
\end{dfn}
An example of a $pq$-extended tree is given in Figure \ref{fig:pq_flow}-(b).

\begin{dfn}[$pq$-Gram]
Let ${\bf T}$ be a tree, ${\bf T}^{p,q}$ the corresponding extended tree, $p > 0$, $q > 0$.
A subtree of ${\bf T}^{p,q}$ is a {\em $pq$-gram} ${\bf G}$ of ${\bf T}$ iff
(i) ${\bf G}$ has $q$ leaf nodes and $p$ non-leaf nodes, 
(ii) all leaf nodes of ${\bf G}$ are children of a single node, 
and (iii) the leaf nodes of ${\bf G}$ are consecutive siblings in ${\bf T}^{p,q}$.
\end{dfn}
The set of $pq$-grams are extracted by sliding two windows horizontally and vertically over a $pq$-extended tree. Window sizes are $p$ and $q$, respectively~\cite{Augsten:2008:PGD:1670243.1670247} (See Figure \ref{fig:pq_flow}-(c)).
The number of nodes of a $pq$-gram is always $p + q$.
An example of extracted $pq$-grams is given in Figure \ref{fig:pq_flow}-(d).

\subsection{$pq$-gram distance}

We can define a distance measure between trees based on $pq$-grams, which we call the {\em $pq$-gram distance}.
Intuitively, the $pq$-gram distance is the number of $pq$-grams that are not shared by two trees.
The $pq$-gram distance is computed as follows (1) extract all $pq$-grams of input trees, and (2) count the number of non-shared $pq$-grams.
To save space, we use the tuple representation of $pq$-grams.

\begin{dfn}[Label Tuple]
Let ${\bf G}$ be a $pq$-gram with nodes $\{v_1, $ $\ldots, $ $v_p, $$v_{p+1}, $$\ldots $$v_{p+q} \}$ where
$v_i$ is the $i$-th node in the preorder transversal of ${\bf G}$. 
Tuple 
$\lambda^*({\bf G}) $$=$$ (\lambda(v_1), $$\ldots, $$\lambda(v_p), $$\lambda(v_{p+1}), $$\ldots $$\lambda(v_{p+q}) )$
is called the {\em label tuple} of ${\bf G}$ where $\lambda(v_i)$ is the label of node $v_i$.
\end{dfn}

\begin{dfn}[$pq$-Gram Index]
Let $\mathcal{A}$ be the multiset of all $pq$-grams of a tree ${\bf T}$, $p, q \in \mathbb{N}$.
The {\em $pq$-gram index}, $\mathcal{I}^{p,q}({\bf T})$, of ${\bf T}$ is defined as the multiset of label tuples of all $pq$-grams of ${\bf T}$,
i.e., $\mathcal{I}^{p,q}({\bf T}) = \bigcup_{{\bf G} \in \mathcal{A}} \lambda^* ({\bf G})$.
\end{dfn}

Figure \ref{fig:pq_flow}-(e) shows an example of the label tuple and the $pq$-gram index.
The size of the $pq$-gram index is linear in the number of tree nodes~\cite{Augsten:2008:PGD:1670243.1670247}.
Hereafter, if the distinction is clear from the context, we use the term $pq$-gram
for both the $pq$-gram itself and its representation as a label tuple.

\begin{dfn}[$pq$-Gram Distance]
Let ${\bf T}_1$ and ${\bf T}_2$ be trees, $\mathcal{I}_1 = \mathcal{I}^{p,q}({\bf T}_1), \mathcal{I}_2 = \mathcal{I}^{p,q}({\bf T}_2)$, $p,q \in \mathbb{N}$. The {\em $pq$-gram distance}, $dist^{p,q}({\bf T}_1, {\bf T}_2)$, between the trees ${\bf T}_1, {\bf T}_2$ is defined as the size of the symmetric difference between their indexes:
\begin{align}
	    dist^{p,q}({\bf T}_1, {\bf T}_2) = |\mathcal{I}_1 \cup \mathcal{I}_2| - 2 |\mathcal{I}_1 \cap \mathcal{I}_2|,
\end{align}
where $\cup$ is multiset union and $\cap$ is multiset intersection. 
\end{dfn}

\begin{ex}
Consider the $1,2$-gram distance between input trees ${\tt a(b,c)}$ and ${\tt a(c,b)}$, namely $dist^{1,2}({\tt a(b,c), a(c,b)})$.
The $pq$-gram indexes for the input trees are
$\mathcal{I}^{1,2}({\tt a(b,c)}) =$ $\{ \tt (a, \ast, b),$ $\tt (a, b, c),$ $\tt (a, c, \ast),$ $\tt (b, \ast, \ast),$ $\tt (c, \ast, \ast) \} $, and
$\mathcal{I}^{1,2}({\tt a(c,b)}) =$ $\{ \tt (a, \ast, c),$ $\tt (a, c, b),$ $\tt (a, b, \ast),$ $\tt (b, \ast, \ast),$ $\tt (c, \ast, \ast) \} $, respectively.
We have $|\mathcal{I}^{1,2}({\tt a(b,c)}) \cup \mathcal{I}^{1,2}({\tt a(c,b)})| = 5 + 5 = 10$, 
and $|\mathcal{I}^{1,2}({\tt a(b,c)}) \cap \mathcal{I}^{1,2}({\tt a(c,b)})| = 2$, 
then $dist^{1, 2}({\tt a(b,c)}, {\tt a(c,b)}) = 10 - 2 \times 2 = 6$.
\end{ex}

Augsten~\cite{Augsten:2008:PGD:1670243.1670247} showed that
the $pq$-gram distance is {\em pseudo-metric}, that is, the distance can be zero for distinct trees in contrast to a normal metric. The computation time of the $pq$-gram distance is $\mathcal{O}(n \log n)$ for fixed $p,q \in \mathbb{N}$, where $n$ is the number of nodes of input trees. 
Moreover, the $pq$-gram distance approximates edit distance weighted by the number of each node, which is called fanout weighted tree edit distance~\cite{Augsten:2008:PGD:1670243.1670247}.

\subsection{Metric learning with nearest neighbors}
\begin{figure}[t]
    \centering
    \includegraphics[width=\linewidth]{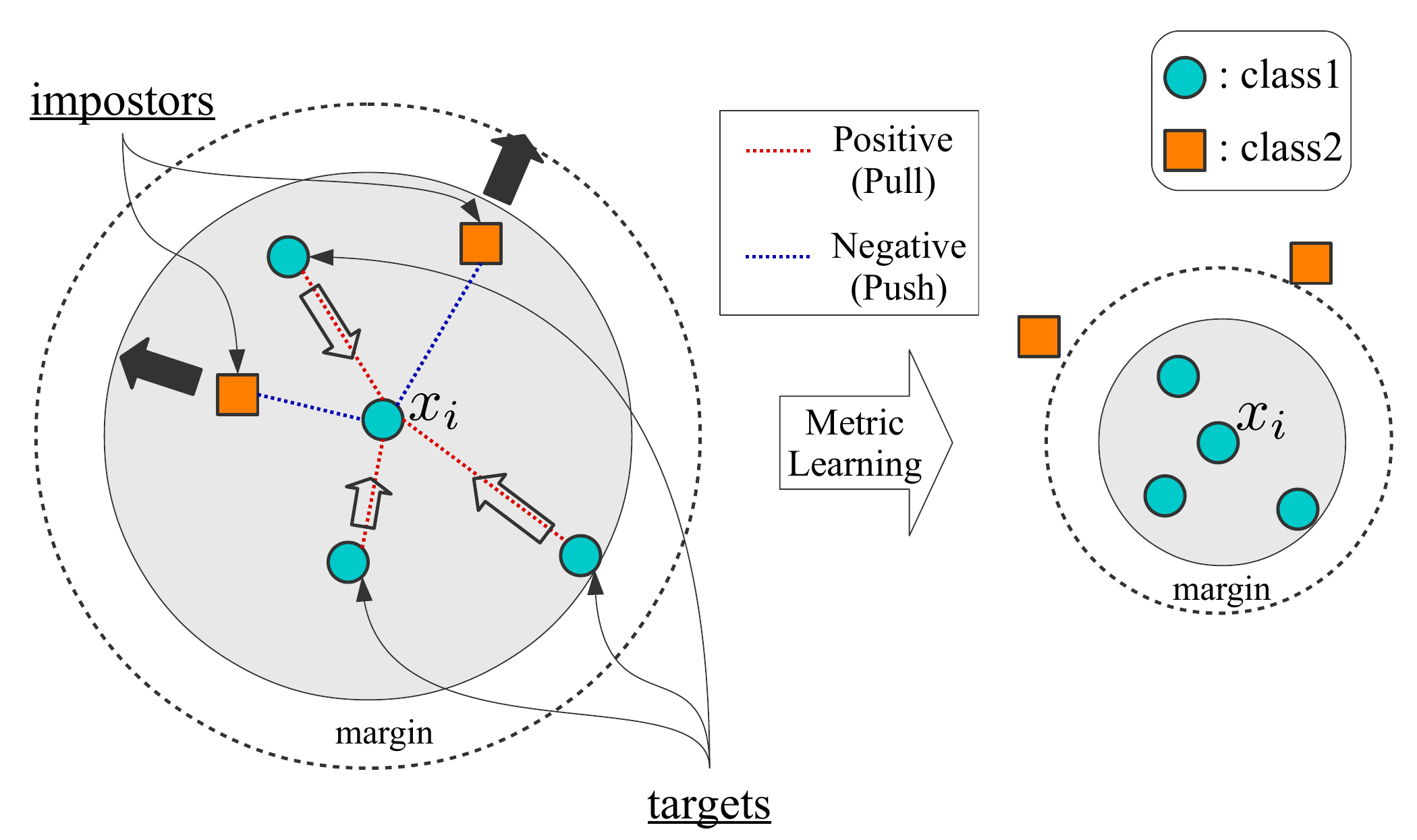}
    \vspace{-1em}
    \caption{An illustration of the LMNN distance learning scheme. For each data point, $x_i$, the neighbors with the same label, called ``targets'', are paired as positive. On the other hand, the neighbors with different labels, called ``impostors'', are paired as negative. After the learning step, targets are brought closer to $x_i$ while impostors are moved further away from $x_i$.}
    \label{fig:lmnn}
\end{figure}

The purpose of metric learning is to adapt the parameterized measure by given positive examples and negative examples.
The metric learning problem is typically formulated as an optimization problem that has the following general form~\cite{DBLP:journals/corr/BelletHS13}:
\begin{align}
    \min_\theta L(d_{\bf \theta}, \mathcal{P}, \mathcal{N})
\end{align}
where $L$ is a loss function that incurs a penalty when training constraints are violated, $d_\theta$ is a distance function parameterized by $\theta$, $\mathcal{P}$ is a set of positive pairs, and $\mathcal{N}$ is a set of negative pairs.
More precisely.
\begin{align}
    \mathcal{P} &= \{(x_i, x_j) : x_i ~\mbox{and} ~ x_j ~\mbox{should be similar}\},\\
    \mathcal{N} &= \{(x_i, x_j) : x_i ~\mbox{and} ~ x_j ~\mbox{should be dissimilar}\}.
\end{align}

Large Margin Nearest Neighbors (LMNN)~\cite{weinberger2009distance} is one of the most widely-used distance learning schemes.
LMNN locally defines the training pairs:
for each data point, same class neighbors, the {\em target neighbors}, are paired as positive,
while different class neighbors, {\em impostors}, are paired as negative.
A schematic illustration of LMNN is given in Figure \ref{fig:lmnn}.

\section{METHOD}
In this section, we introduce a novel approach of metric learning between trees based on $pq$-grams.
The $pq$-gram distance is computed by simply counting the number of $pq$-grams not shared between input trees.
Some $pq$-grams, however, can be important as discriminators for a given classification problem.
We introduce a weight function for $pq$-grams and learn it from examples.

Our approach shares some aspects with edit distance-based approaches~\cite{Bellet2012, Mokbel2015306, pmlr-v80-paassen18a}.
The edit distance is defined by the minimum number of edit operations needed to transform one tree into another. 
Edit distance-based methods primarily learn the cost of edit operations, i.e., they distinguish which edit operations are essential for a given classification problem.
The edit operations are defined between two nodes, so they mostly learn the importance of the relations of nodes.
On the other hand, our approach learns the importance of subtrees of the tree structure.

\subsubsection{$pq$-gram distance with vector representations}
Augsten~\cite{Augsten:2008:PGD:1670243.1670247} represents the $pq$-gram index as a multiset
and compute the distance by operations between multisets.
However, if the set of tree node labels $\Sigma$ is finite, 
the $pq$-gram index can be represented as a vector of fixed dimension.
Vector representation allows us to compute the $pq$-gram distance efficiently.

\begin{dfn}{($pq$-Gram Vector)}
Let $\mathcal{J}$ be the set of all $p q$-grams in dataset $\mathcal{D}$.
For tree ${\bf T} \in \mathcal{D}$ and its $pq$-gram index $\mathcal{I}^{p, q}({\bf T})$,
{\em $p q$-gram vector} is 
a $|\mathcal{J}|$-dim count vector ${\bf v}^{p, q}({\bf T})$.
Each dimension of ${\bf v}^{p, q}({\bf T})$ corresponds to a $pq$-gram. 
\end{dfn}

In order to compute the $pq$-gram distance from counting vectors, 
we introduce a function that computes the number of not shared $pq$-grams.
\begin{dfn}{($pq$-Gram Symmetric Difference Vector)}
Let ${\bf T}_1$ and ${\bf T}_2$ be input trees.
We define the {\em $pq$-gram symmetric difference vector} between ${\bf T}_1$ and ${\bf T}_2$ as:
\begin{align}
    {\bf d}^{p,q}({\bf T}_1, {\bf T}_2) = {\bf v}^{p,q}({\bf T}_1) &+ {\bf v}^{p,q}({\bf T}_2) \notag \\  &- 2 {\bf min}({\bf v}^{p,q}({\bf T}_1), {\bf v}^{p,q}({\bf T}_2)),
\end{align}
where ${\bf min}$ is the element-wise minimum function, i.e., for $n$-dimensional vector ${\bf x}$ and ${\bf y}$,
\begin{align}
    {\bf min}({\bf x}, {\bf y}) = \left( \min(x_i, y_i) \right)_{i=1 \ldots n}.
\end{align}
\end{dfn}

\begin{thm}
Let ${\bf T}_1$ and ${\bf T}_2$ be trees.
The $pq$-gram distance equals the sum of the elements of the $pq$-gram symmetric difference vector, i.e.,
\begin{align}
    dist^{p,q}({\bf T}_1, {\bf T}_2) = {\bf 1}^\mathrm{T} \cdot {\bf d}^{p,q}({\bf T}_1, {\bf T}_2),
\end{align}
where ${\bf 1}$ is the all-one vector and ${\bf d}^{p,q}({\bf T}_1, {\bf T}_2)$ is the $pq$-gram symmetric difference vector.
\end{thm}

\subsection{Computing the weighted $pq$-gram distance}
In this section, we introduce the {\em weighted $pq$-gram distance} to perform metric learning based on $pq$-grams.
The weight reflects the ``importance'' of each $pq$-gram, and allows the $pq$-gram distance to yield highly granular classification.
In order to make the distance function differentiable and always positive, we use the softplus function.
It enables us to learn the distance function by gradient descent techniques and retain the triangle inequality during the learning process.

\subsubsection{Weighted $p q$-gram distance}
\begin{dfn}[Softplus function]
The softplus function is defined as:
\begin{align}
    \mathit{softplus}(x) = \ln(1 + e^x).
\end{align}
The softplus function always returns positive values. 
In order to prevent weight parameters from being negative, we apply the softplus function to the parameters.
The softplus function is differentiable with respect to the input variables, and enables us to learn distance parameters
by gradient descent techniques.
\end{dfn}
\begin{dfn}[Weighted $pq$-Gram Distance]
Let ${\bf T}_1$ and ${\bf T}_2$ be input trees.
The {\em weighted $pq$-gram distance}, $dist^{p,q}_{{\bf w}}({\bf T}_1, {\bf T}_2)$ is defined as follows:
\begin{align}
	    dist^{p,q}_{{\bf w}}({\bf T}_1, {\bf T}_2) = {\bf a}({\bf w})^\mathrm{T} \cdot {\bf d}^{p,q}({\bf T}_1, {\bf T}_2),
\end{align}
where
${\bf a}({\bf w}) = (\mathit{softplus}(w_i))_{i=1 \ldots |{\bf w}|} =  \left( \ln \left(1 + e^{ w_i} \right) \right)_{i=1 \ldots |{\bf w}|}$.\\
${\bf w}$ is a parameter that we learn. 
${\bf d}^{p,q}({\bf T}_1, {\bf T}_2)$ is the $pq$-gram symmetric difference vector between ${\bf T}_1$ and ${\bf T}_2$.
\end{dfn}

\begin{thm}
The weighted $p q$-gram distance is pseudo-metric, i.e., satisfies the following conditions:\\
(i) non-negativity: $dist^{p,q}_{{\bf w}}(x,y) \geq 0$\\
(ii) reflexivity: $x = y \Rightarrow dist^{p,q}_{{\bf w}}(x,y)=0$\\
(iii) symmetry: $dist^{p,q}_{{\bf w}}(x,y) = dist^{p,q}_{{\bf w}}(y,x)$\\
(iv) triangle inequality: $dist^{p,q}_{{\bf w}}(x,y) + dist^{p,q}_{{\bf w}}(y,z)  \geq dist^{p,q}_{{\bf w}}(x,z)$
\end{thm}

\begin{proof}
(i), (ii), and (iii) are clear by definition.\\
(iv)
Let $\mathcal{A}_x$, $\mathcal{A}_y$, and $\mathcal{A}_z$ be the multisets of extracted $pq$-grams of trees $x$, $y$, and $z$, respectively.
Let $g$ be a $pq$-gram with $g \in \mathcal{A}_x \cup \mathcal{A}_y \cup \mathcal{A}_z$,
and $i$, $j$, and $k$ be the numbers of occurrences of $g$ in $\mathcal{A}_x$, $\mathcal{A}_y$, and $\mathcal{A}_z$, respectively.
Let $d_g(x, y)$, $d_g(y, z)$, $d_g(x, z)$ be the contributions of $g$ to the distance $dist^{p,q}_{{\bf w}}(x,y)$, $dist^{p,q}_{{\bf w}}(y,z)$, and  $dist^{p,q}_{{\bf w}}(x,z)$, respectively. The distance is computed as the sum of the contributions: $dist^{p,q}_{{\bf w}}(x,y) = \sum_{g \in \mathcal{A}_x \cup \mathcal{A}_y \cup \mathcal{A}_z} d_g(x, y)$.
Here, $ d_g(x,y) = a_g(i + j - 2 \min(i, j)) = a_g |i - j|$, $d_g(y,z) = a_g |j - k|$, and $d_g(x,z) = a_g |i - k|$, where $a_g$ is a weight parameter for $g$. 
Note that $a_g$ is positive since it is output by the softplus function.
Therefore, we have $dist^{p,q}_{{\bf w}}(x,y) + dist^{p,q}_{{\bf w}}(y,z) = \sum_{g \in \mathcal{A}_x \cup \mathcal{A}_y \cup \mathcal{A}_z} (d_g(x, y) +  d_g(y,z)) = \sum_{g \in \mathcal{A}_x \cup \mathcal{A}_y \cup \mathcal{A}_z} a_g (|i-j| + |j-k|) \geq \sum_{g \in \mathcal{A}_x \cup \mathcal{A}_y \cup \mathcal{A}_z} a_g |i-k| = dist^{p,q}(x,z)$.
\end{proof}

\subsection{Learning the weighted $pq$-gram distance}
The steps to perform metric learning are 
(1) create training pairs from a given dataset,  
(2) set an appropriate loss function for the training pairs,
and (3) optimize the loss function with respect to the parameter of the target distance function.

We generate training pairs following the LMNN scheme.
For every data point $x_i$ with class label $y_i$ in given dataset $\mathcal{D}$, we create positive and negative pairs as follows:
\begin{align}
    \mathcal{P}_i &= \left\{(x_i, x_j): x_j  \in N^{+}_k(x_i) \right\}, \label{eq:pos_pair}\\
    \mathcal{N}_i &= \left\{(x_i, x_j): x_j  \in N^{-}_k(x_i) \right\}, \label{eq:neg_pair}
\end{align}
where $N^{+}_k(x_i)$ is the ``target'' function that returns the set of $k$-nearest neighbors  with the same label $y_i$,
and $N^{-}_k(x_i)$ is an ``impostor'' function that returns the set of data points with different labels and are closer than the farthest ($k$-th) target.
We now define the set of positive pairs as $\mathcal{P} = \bigcup_{i=1 \ldots m} \mathcal{P}_i$, and the set of negative pairs as $\mathcal{N} = \bigcup_{i=1 \ldots m} \mathcal{N}_i$,
where $m$ is the number of training data points. We note that the training pairs are created with the default distance metric before the learning step.

Now we introduce the loss function based on the hinge loss:
\begin{align}
    \label{eq:loss}
   L(dist^{p,q}_{\bf w}, \mathcal{P}, \mathcal{N}) =  & \beta ||{\bf w}||^2 + \sum_{({\bf T}_1, {\bf T}_2) \in \mathcal{P} } [dist^{p,q}_{{\bf w}}({\bf T}_1, {\bf T}_2) - \mu_1]_+ \notag \\ &+ \sum_{({\bf T}_1, {\bf T}_2) \in \mathcal{N} } [\mu_2 - dist^{p,q}_{{\bf w}}({\bf T}_1, {\bf T}_2) ]_+,
\end{align}
where $\beta$ is a regularization coefficient, $[x]_+ = \max(0, x)$, and   $\mu_1$ and $\mu_2$ are constants that represent margins. 
The first term is the L2 regularization term, the second term is a loss that makes positive pairs closer, and the third term is a loss that makes negative pairs farther apart.
The hinge loss-based formulation is widely-used in margin-based methods such as soft-margin SVM~\cite{Chen04}, LMNN~\cite{weinberger2009distance}, and GESL~\cite{Bellet2012}.
If a positive (resp. negative) pair satisfies a certain criterion, i.e., it is close (resp. far) enough, then it does not contribute to the loss function. 

We minimize the loss function by gradient descent. In order to perform gradient descent, we need the gradient $\nabla_{\bf w} dist^{p,q}_{\bf w}({\bf T}_1, {\bf T}_2)$
for input trees ${\bf T}_1$ and ${\bf T}_2$.
The gradient of weighted $p q$-gram distance function with respect to ${\bf w}$ is computed as follows:
\begin{align}
    \nabla_{\bf w} dist^{p,q}_{\bf w}({\bf T}_1, {\bf T}_2) = \left(\frac{ e^{ w_i}}{1 + e^{ w_i}} d^{p,q}_i({\bf T}_1, {\bf T}_2) \right)_{i= 1 \ldots |{\bf w}|}, 
\end{align}
where $w_i$ is the $i$-th element of ${\bf w}$ and $d^{p,q}_i({\bf T}_1, {\bf T}_2)$ is the $i$-th element of  the $pq$-gram symmetric difference vector ${\bf d}^{p,q}({\bf T}_1, {\bf T}_2)$.

The learning procedure is summarized as follows:
First, the $pq$-gram vector representations are computed for all input trees.
We note that the computed $pq$-gram vector representations are used internally by the distance function $dist^{p,q}_{\bf w}$ in the following steps.
Second, the set of positive pairs and the set of negative pairs are created by Eq. (\ref{eq:pos_pair}) and Eq. (\ref{eq:neg_pair}), respectively, from a given training dataset.
Finally we minimize the loss function in Eq. (\ref{eq:loss}) with respect to ${\bf w}$ by gradient descent.
In practice, we update impostors during the learning process to improve model performance.

\section{EXPERIMENTS}

In this section, we discuss our experiments on artificial and real-world datasets.
The experiments are designed to show that
the proposed approach not only achieves competitive results with state-of-the-art edit distance-based methods in various classification problems, but also solves the classification problems much faster than edit distance-based methods.
All experiments were performed on a desktop computer with 
Intel(R) Xeon(R) E5-2680 v2 @ 2.80GHz CPU,
126GB RAM, and CentOS Linux 7.4.

We used our Python 3.6 and PyTorch~\cite{pytorch} implementation for LMNN learning with the weighted $pq$-gram distance. 
The tree edit distance algorithm is implemented following~\cite{paasen18-supp}.
We adopted Adam optimizer~\cite{Adam} to optimize parameters in our training phase.
We also used a Java implementation\footnote{https://pub.uni-bielefeld.de/data/2919994} for BEDL and GESL.

\subsection{Dataset}
We evaluated our approach on one artificial dataset and several real-world datasets.

\begin{figure*}[t]
    \centering
    \includegraphics[width=\linewidth]{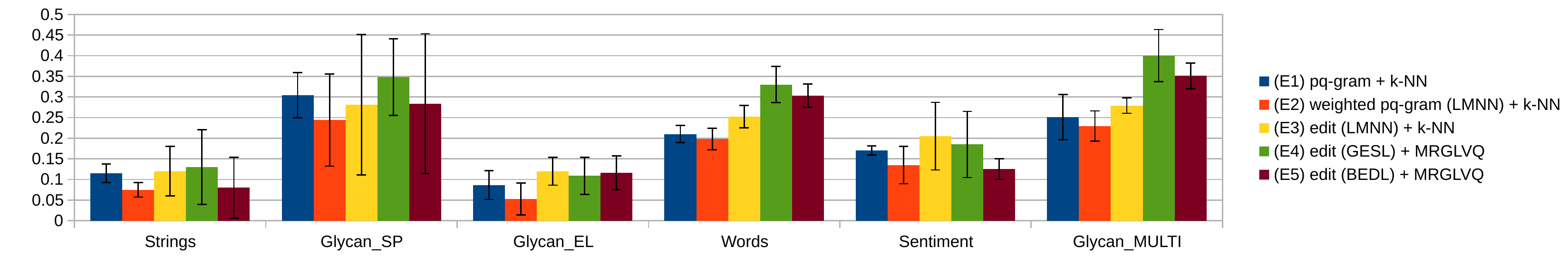}
    \caption{The figure indicates error rates for each real-world datasets. The bottom labels specify the datasets. The bar represents the error rate for each method. The vertical black line represents the standard deviation of the results of the folds. For each dataset, the bar corresponds to each combination of distance and classifier:
(E1) the $pq$-gram distance and the $k$-nearest neighbor classifier, 
(E2) the weighted $pq$-gram distance with LMNN and the $k$-nearest neighbor classifier (proposed),
(E3) the edit distance with LMNN and the $k$-nearest neighbor classifier,
(E4) the edit distance with GESL and the MRGLVQ classifier,
(E5) the edit distance with BEDL and the MRGLVQ classifier.
 }
    \label{fig:experiment1}
\end{figure*}

\paragraph{Strings}
This dataset consists of two classes of strings with a length of exactly nine. 
The first class is drawn randomly from the set of strings expressed by $(({\tt A} | {\tt B})({\tt C} | {\tt D})({\tt A} | {\tt B}))^3$ in the regular expression.
For the second class, we do the same from the set of strings expressed by  $({\tt A}|{\tt B}|{\tt C}|{\tt D})^9$ in the regular expression.
We can observe that substring ${\tt DAD}$ never appears in the first class, and 
every string in the first class is generated in a ``periodic way'' 
unlike those in the second class. These facts are important for classifying these strings.
We can regard a string as a tree (without branching nodes) in a natural way.
We used 100 strings for each class.
\vspace{-1em}
\paragraph{Glycan}
We used two datasets from KEGG database~\cite{10.1093/glycob/cwj010} as used in ~\cite{yamazaki15}.
Glycans are defined as the third major class of biomolecules following DNA and proteins.
Each monosaccharide in a glycan structure is connected to one or
more monosaccharides, and we can regard a glycan structure as a labeled tree.
CarbBank/CCSD~\cite{10.1093/glycob/2.6.505} gives the class labels for glycans.
The trees have node labels and edge labels. We put edge labels into corresponding child nodes in the same way as~\cite{yamazaki15}.
For instance, subtree $\frac{\tt 1b4}{} {\tt Gal} \frac{\tt 1a3}{} {\tt Fuc} $ is represented as ${\tt Gal.1b4 - Fuc.1a3}$.
Every leaf node is represented by the special label ${\tt \$}$.
Each glycan structure is assigned to a blood component class among Erythrocyte, Leukemic, Serum, and Plasma.
We created two binary classification problems (i) Erythrocyte/Leukemic (Glycan\_EL) and (ii) Serum/Plasma (Glycan\_SP).
These problems have 138 trees and 267 trees, respectively.
For evaluating multi-label classification, we also used the dataset that contains all instances (Glycan\_MULTI) in accuracy comparison.
Glycan\_MULTI contains 405 trees and 4 class labels.
\vspace{-1em}
\paragraph{Words}
The words dataset used in  \cite{Bellet2012} contains English words and French words extracted from Wiktionary\footnote{http://en.wiktionary.org/wiki/Wiktionary:Frequency\_lists}. 
It consists of basic English/French words in order of frequency of use.
We can regard a word as a tree (without branching nodes) in a natural way.
Every leaf node is represented by the special label ${\tt \$}$.
We considered only words of length at least 4 to remove articles and prepositions.
We used the top 500 words for each class.
\vspace{-1em}
\paragraph{Sentiment Treebank}
Sentiment Treebank dataset contains movie reviews with their parse trees.
The internal nodes have one of the 5-class labels from highly negative to highly positive.
We set the class label as the sentence is positive or negative as a whole, i.e., the root node's label. 
Every root node is replaced by a unique node whose label is ${\tt root}$. 
Every leaf node represents a specific word in the review sentence.
We replaced each word with its POS tag for scaling.
We randomly chose 100 trees for each class. 

The datasets we used are summarized in Table \ref{tab:datasets}.
\begin{table}
    \centering
    \begin{tabular}{c|c c c | c}
    Dataset & trees & node labels & mean tree size & class\\ \hline
    Strings & 200 & 4 & 9.0 & 2\\
    Glycan\_SP &  138 & 48 & 10.9 & 2\\
    Glycan\_EL &  267 & 37 & 13.4 & 2\\
    Words& 1000 & 28 & 6.8 & 2\\
    Sentiment & 200 & 39 & 39.2 & 2\\
    Glycan\_MULTI & 405 & 51 & 12.5 & 4\\\hline
    \end{tabular}
    \caption{The description of datasets used in our experiments. The first column shows the number of trees in each dataset. The second column shows the number of node labels of trees in each dataset. The third column shows the mean size of trees (number of nodes in a tree) for each dataset. The fourth column shows the number of classes.}
    \label{tab:datasets}
    \vspace{-1.5em}
\end{table}

\subsection{Accuracy comparison}
We evaluated several classification problems using different models with different distance measures.
The problem setting has 3 parts: 
(i) the distance measure used by the classification model, 
(ii) the metric learning algorithm, 
and (iii) the distance-based classification model.
We compared the following 5 settings:
(E1) the $pq$-gram distance and the $k$-nearest neighbor classifier, 
(E2) the weighted $pq$-gram distance with LMNN and the $k$-nearest neighbor classifier (proposed),
(E3) the edit distance with LMNN and the $k$-nearest neighbor classifier,
(E4) the edit distance with GESL and the MRGLVQ classifier~\cite{Bellet2012,NEBEL2015295},
and (E5) the edit distance with BEDL and the MRGLVQ classifier~\cite{NEBEL2015295, pmlr-v80-paassen18a}.

On each data set, we performed 5-fold cross-validation and compared the mean test error across the folds.
In the setting (E1) and (E2), we set $p=2$, $q=2$ as $pq$-gram size.
We set $k=1$ for Strings and Glycan\_MULTI dataset, and $k=3$ for others, where $k$ is the number of neighbors for the $k$-nearest neighbor classifier and the number of ``targets'' of LMNN learning.
As these parameters affect the classification results, we analyze their impact in the next subsection.
Other parameters were set as follows: $\mu_1 = \mu_2 = 5.0$ for margin parameters, $\eta = 10^{-2}$ for the initial learning rate of the Adam optimizer~\cite{Adam}, $\beta = 10^{-4}$ for the L2 regularization.
We trained the model for 600 epochs.
We updated impostors for LMNN every 50 epochs.
In the LMNN learning step, we randomly chose 200 training data points if the number of input training data points is more than 200. 
In the setting (E3), we fixed the optimal edit operations that are computed for the first time of the learning algorithm, in the same way, as~\cite{DBLP:journals/corr/BelletHS13}
and performed metric learning using the LMNN scheme with gradient descent with respect to the edit costs.
The parameters for the settings (E4) and (E5) are selected by nested cross-validation following \cite{pmlr-v80-paassen18a}.

Figure \ref{fig:experiment1} shows the results of our experiments.
On each dataset, the $k$-nearest neighbor classifier with the weighted $pq$-gram distance (E2) achieves a lower error rate than the $pq$-gram distance (E1) .
Moreover, for all datasets, the $k$-nearest neighbor classification with the weighted $pq$-gram distance (E2) achieves better results than that with the tree edit distance (E3).
Also, our approach achieves competitive results with the state-of-the-art edit distance-based methods such as GESL and BEDL with the MRGLVQ classifier (E4) (E5).

\subsubsection{Effects of parameters}
\begin{figure}[t]
  \begin{center}
    \begin{tabular}{c}
      \begin{minipage}{0.5\hsize}
        \begin{center}
                          {\footnotesize  Strings}
          \includegraphics[clip, width=\linewidth]{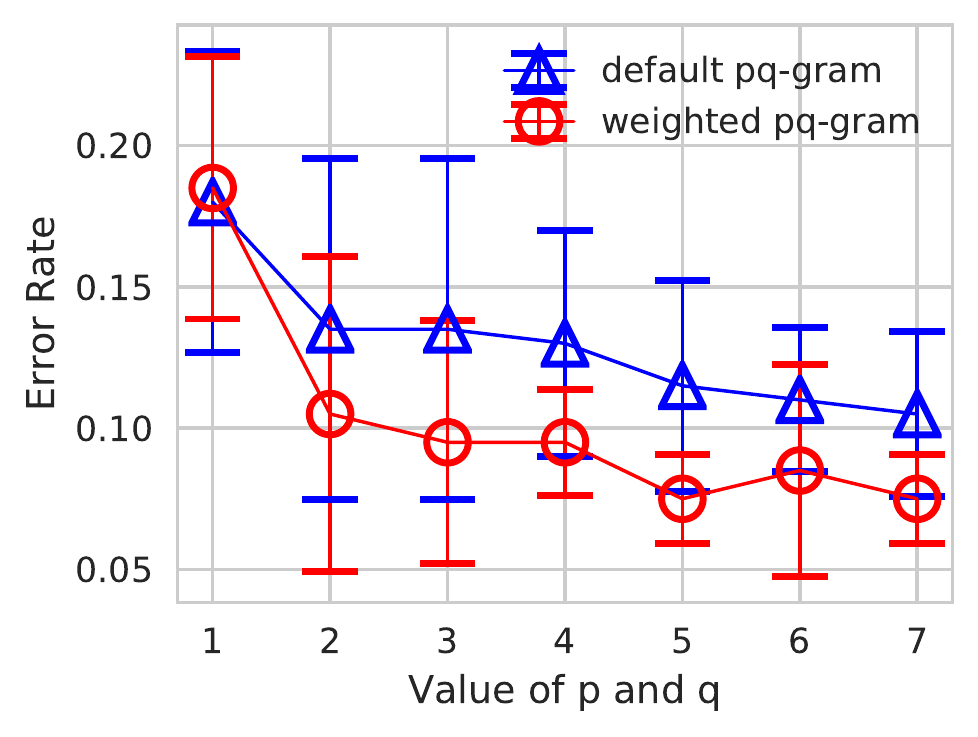}
        \end{center}
      \end{minipage}
      \begin{minipage}{0.5\hsize}
        \begin{center}
                          {\footnotesize  Glycan\_SP}
          \includegraphics[clip, width=\linewidth]{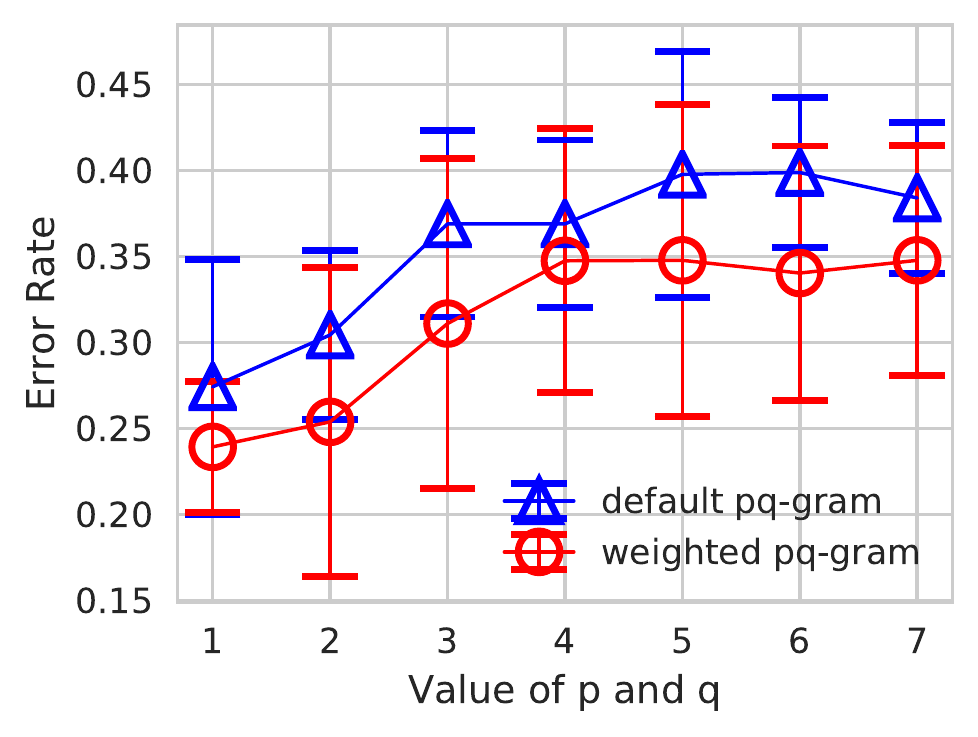}
        \end{center}
      \end{minipage}
    \end{tabular}
  \end{center}
  \vspace{-1.0em}
  \caption{The figure indicates error rates with standard deviations for $1 \le p=q \le 7$. The horizontal axis represents the value of $p$ and $q$. The vertical axis represents the error rate of $k$-nearest neighbor classifier.
  The blue line represents the error rate with the $pq$-gram distance, and the red line represents the error rate with the weighted $pq$-gram distance.}
    \label{fig:pqvalues}
  \vspace{-1.0em}
\end{figure}
\begin{figure}[t]
  \begin{center}
    \begin{tabular}{c}
      \begin{minipage}{0.5\hsize}
        \begin{center}
                {\footnotesize  Strings}
          \includegraphics[clip, width=\linewidth]{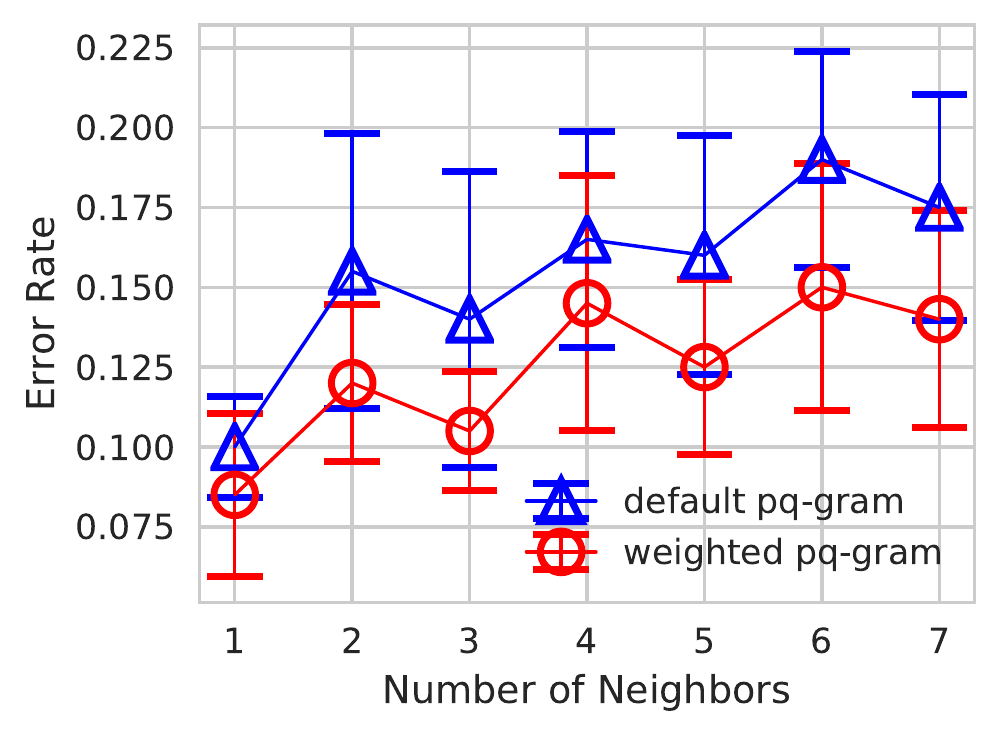}
        \end{center}
      \end{minipage}
      \begin{minipage}{0.5\hsize}
        \begin{center}
                          {\footnotesize  Glycan\_SP}
          \includegraphics[clip, width=\linewidth]{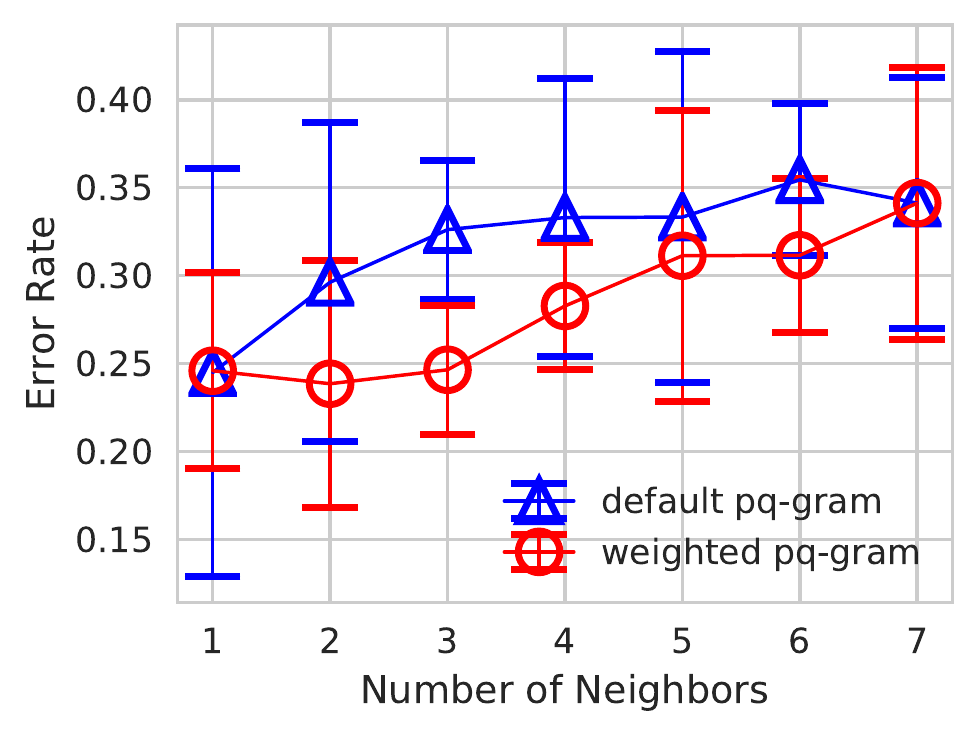}
        \end{center}
      \end{minipage}
    \end{tabular}\\
  \end{center}
  \vspace{-1.0em}
\caption{The figure indicates error rates with standard deviations for $1 \le k \le 7$. The horizontal axis represents the number of neighbors for the $k$-nearest neighbor classifier and the LMNN learning algorithm. The vertical axis represents the error rate of the $k$-nearest neighbor classifier. The blue line represents the error rate with the $pq$-gram distance, and the red line represents the error rate with the weighted $pq$-gram distance.}
  \vspace{-1.0em}
        \label{fig:kvalues}
\end{figure}
Since the values of $p$ and $q$ determine the shape of $pq$-grams, they can affect the error rates of classification.
We analyze the effect on our proposed method of changing the values of $p$ and $q$ (E2).
In particular, we investigate the error rate of the 3-nearest neighbor classifiers with
the conventional $pq$-gram distance and with the weighted $pq$-gram distance.
The values $p$ and $q$ are chosen as $1 \le p = q\le 7$.
Figure \ref{fig:pqvalues} shows the transition in error rates with respect to $p$ and $q$.
For the Strings dataset,
the case $p = q = 1$ is worst among all, which follows our intuition that every string in the first class is composed of substrings of length three, and these substrings can be seen only in the case where $p = q \ge 2$.
For the Glycan\_SP dataset, the error rates gradually increase as both $p$ and $q$ become large.
In both datasets, for all $p$ and $q$ values except for the case $p=q=1$ in the Strings dataset, the $k$-nearest neighbor classifier with the weighted $pq$-gram distance 
outperformed that with the conventional $pq$-gram distance.

We also analyze the effect of the number of neighbors $k$ on the $k$-nearest neighbor classifier and the LMNN metric learning scheme.
In the training step, we created training pairs with $k$ targets for each tree. 
In the classification step, we performed the $k$-nearest neighbor classification.
Figure \ref{fig:kvalues} shows the transition in error rate with respect to $k$.
Interestingly, in the Strings dataset, the case of $k=1$ achieves the highest accuracy among all.
With regard to the number of neighbors $k$ for the $k$-nearest neighbor classifier,
Hastie \cite{TESL} pointed out that the best $k$ value is situation dependent.
We highlight the fact that our proposed distance outperformed the conventional $pq$-gram distance regardless of the $k$ values in both datasets except for the case $k=1, 7$ in the Glycan\_SP dataset.

\subsubsection{Interpretability}
In this subsection, we discuss the interpretability of our method.
We can consider that $pq$-grams with substantial weights are important discriminators for classification problems 
since they essentially determine the classification results.

We exhibit some $pq$-grams receiving substantial weights and their occurrences in each class in Strings, Glycan\_EL, and Words dataset as in Table \ref{tab:top_grams}. 
We can observe that high-weight $pq$-grams tend to appear many times in one of the classes, but not so much in the other classes.
This fact implies that $pq$-grams with substantial weights are important features for classifying trees.
In the Glycan\_EL dataset, for example,
$\tt (GlcNAc.1b4, Gal.1b3, \$, *)$ appears only with class 1. 
It means that subtree ${\tt  \frac{1b4}{} GlcNAc \frac{1b3}{} Gal }$ which contains the leaf node, 
is a key feature for class 1.
In the Words dataset, $pq$-gram ${\tt (e, u, r,*)}$ appears on French words 14 times, but never on English words,
which means French words in the dataset often have substring ``eur'', but English words do not.

\begin{table}
    \centering
    \begin{tabular}{c|c|c|c}
        Dataset&$pq$-gram & class1 & class2\\ \hline\hline
&$\tt (B, B, A, *)$&0&9\\
&$\tt (D, A, *, *)$&21&8\\
Strings&$\tt (A, D, B, *)$&45&12\\
&$\tt (C, A, *, *)$&26&2\\
&$\tt (D, A, D, *)$&0&11\\\hline
&$\tt (*, Xyl., Gal.1b4, *)$&3&0\\
&$\tt (GlcNAc.1b4, Gal.1b3, \$, *)$&2&0\\
Glycan\_EL&$\tt (Gal.1b4, GlcNAc.1b3, \$, *)$&4&0\\
&$\tt (Gal., GlcNAc.1b3, \$, *)$&0&1\\
&$\tt (GalNAc.1b3, GalNAc.1a3, *, \$)$&2&0\\\hline
&$\tt (c, h, *, \$)$&6&0\\
&$\tt (o, n, *, e)$&9&0\\
Words&$\tt (s, e, u, *)$&0&4\\
&$\tt (e, u, r, *)$&0&14\\
&$\tt (o, i, s, *)$&0&6\\\hline
    \end{tabular}
    \caption{Extracted high-weight $pq$-grams and appearance number in each class. Right 2 columns show the number of occurrences of $pq$-grams in each class.}
    \label{tab:top_grams}
    \vspace{-2.5em}
\end{table}

\subsection{Running time comparison}

In order to show the practical performance of our method,
we compare the running times of classification algorithms based on the weighted $pq$-gram distance and the tree edit distance.
In this experiment, we run the standard $k$-nearest neighbor algorithm with two distinct distance functions, the weighted $pq$-gram distance and the tree edit distance.
We measured the time for executing whole process: (i) encoding into count vectors, (ii) computing the distances between the test data and the training data to identify neighbors, (iii) and making an inference by the majority vote. The first encoding step is executed only for the weighted $pq$-gram distance.
We first note that the theoretical running time $\mathcal{O}(n \log n)$ of computing the $pq$-gram distance between trees is much faster than 
the fastest known tree edit distance algorithm, that of \cite{Demaine:2009:ODA:1644015.1644017} which runs in $\mathcal{O}(n^3)$ time.
Thus, the running time of the $k$-nearest neighbor inference with the weighted $pq$-gram distance is $\mathcal{O}(m_1 m_2 (n \log n + k))$,
and that with the tree edit distance is $\mathcal{O}(m_1 m_2 (n^3 + k))$, where $m_1$ is the number of training data points and $m_2$ is the number of test data points.

Table \ref{tab:time} shows the mean running time of 3-nearest neighbor inference for 5-fold cross-validation.
In all datasets, the weighted $pq$-gram distance yields much shorter inference times than the tree edit distance as a distance function.
Especially in the Sentiment dataset, whose mean tree size is the largest among all datasets, 
the $k$-nearest neighbor classifier using our proposed method yields over 5000 times shorter inference time than that using the tree edit distance.


We note that edit distance-based methods such as GESL do not directly learn the edit cost of tree edit distance. 
They first compute an optimal sequence of edit operations for training data which is then fixed (held unchanged). 
Then, they learn an appropriate edit cost along with this sequence of edit operations, 
which makes their learning process much easier and faster than directly learning the edit cost.
However, when making an inference for test data using the learned distance, the costly computation for the edit distance is still a major drawback.

\begin{table}
    \centering
    \begin{tabular}{c|c c}
    Dataset& proposed & edit distance \\ \hline\hline
    Strings & \textbf{1.29 $\pm$ 0.098} sec & 23.7 $\pm$ 0.09 sec\\\hline
    Glycan\_SP & \textbf{0.665 $\pm$ 0.059} sec & 28.4 $\pm$ 2.21 sec\\\hline
    Glycan\_EL & \textbf{0.249 $\pm$ 0.121} sec & 182.9 $\pm$ 8.86 sec\\\hline
    Words& \textbf{41.3 $\pm$ 0.31} sec & 283.9 $\pm$ 5.18 sec\\\hline
    Sentiment & \textbf{1.79 $\pm$ 0.105} sec & 9013 $\pm$ 770 sec\\\hline
    \end{tabular}
    \caption{Mean inference time of the 3-nearest neighbor classifier in each dataset. The first column specifies the dataset. The second column shows the inference time of the 3-nearest neighbor classifier using the weighted $pq$-gram distance. The third column shows that the inference time of the 3-nearest neighbor classifier using the tree edit distance.}
    \label{tab:time}
    \vspace{-2.5em}
\end{table}

\section{CONCLUSION}
This contribution has proposed a novel metric learning approach for tree-structured data that has the following features.
The differentiable parameterized distance based on $pq$-grams (proposed herein), called the {\em weighted $pq$-gram distance}, achieves fast metric learning for tree-structured data.
The time complexity of the weighted $pq$-gram distance is $\mathcal{O}(n \log n)$, while that of the tree edit distance is more than $\mathcal{O}(n^3)$, where $n$ is the number of nodes of the input trees.
Moreover, computation of the proposed distance involves only basic vector operations with the softplus function.
It enables the distance function to be learned by the gradient descent techniques while retaining the triangle inequality during the learning process.
Second, a way of learning the weighted $pq$-gram distance through LMNN, which is one of the most widely-used metric learning schemes, was also proposed.
Moreover, the metric learning problem was formulated as an optimization problem based on the hinge loss-based formulation.
Third, the results of our proposal are interpretable. 
Our weight parameter indicates which substructures in trees are important for classifying input trees.

We have empirically shown that for various classification problems our proposed method reduces error rates compared to the conventional $pq$-gram distance using the $k$-nearest neighbor classification.
Moreover, our approach achieved competitive results with the state-of-the-art edit distance-based methods such as GESL and BEDL.
We have also shown that our approach solves classification problems much faster than the edit distance-based methods.
In our experiments, the $k$-nearest neighbor classifier using our proposed method solved the various classification problems at most 5000 times faster than that using the tree edit distance.

\section*{ACKNOWLEDGEMENTS}
This work was partly supported by JSPS KAKENHI Grant Number 17K19973.


\bibliographystyle{abbrv}
\bibliography{920_paper}

\end{document}